\documentclass{article}
\usepackage{ijcai24}

\usepackage{times}
\usepackage{soul}
\usepackage{url}
\usepackage[hidelinks]{hyperref}
\usepackage[utf8]{inputenc}
\usepackage[small]{caption}
\usepackage{graphicx}
\usepackage{amsmath}
\usepackage{amsthm}
\usepackage{amssymb}
\usepackage{amsfonts}
\usepackage{mathtools}
\usepackage{mathrsfs}
\usepackage{booktabs}
\usepackage{float}
\usepackage{algorithm}
\usepackage{algorithmic}
\usepackage[switch]{lineno}
\usepackage{bm}
\usepackage{makecell}
\usepackage{framed}
\mathtoolsset{showonlyrefs}

\newcommand{\bmx}{\bm{x}}
\newcommand{\bmy}{\bm{y}}
\newcommand{\bmz}{\bm{z}}
\newcommand{\bmr}{\bm{r}}

\newcommand{\bmf}{\bm{f}}
\newcommand{\omm}{OneMinMax}
\newcommand{\lotz}{LeadingOnesTrailingZeroes}

\newcommand{\lo}{\mathrm{LO}}

\newcommand{\tz}{\mathrm{TZ}}

\newtheorem{theorem}{Theorem}

\newtheorem{definition}{Definition}

\title{An Archive Can Bring Provable Speed-ups \\in Multi-Objective Evolutionary Algorithms
}

\author{
	Anonymous
}

\author{
Chao Bian$^1$
\and
Shengjie Ren$^1$\and
Miqing Li$^{2}$\And
Chao Qian$^1$
\affiliations
$^1$National Key Laboratory for Novel Software Technology, Nanjing University, Nanjing 210023, China\\
School of Artificial Intelligence, Nanjing University, Nanjing 210023, China\\
$^2$School of Computer Science, University of Birmingham,
Birmingham B15 2TT, U.K.
\emails
\{bianc, qianc\}@lamda.nju.edu.cn, shengjieren36@gmail.com,
m.li.8@bham.ac.uk
}

\begin{document}

\maketitle

\begin{abstract}
    In the area of multi-objective evolutionary algorithms (MOEAs), there is a trend of using an archive to store non-dominated solutions generated during the search. This is because 1) MOEAs may easily end up with the final population containing inferior solutions that are dominated by other solutions discarded during the search process and 2) the  population that has a commensurable size of the problem's Pareto front is often not practical. In this paper, we theoretically show, for the first time, that using an archive can guarantee speed-ups for MOEAs. Specifically, we prove that for two well-established MOEAs (NSGA-II and SMS-EMOA) on two commonly studied problems (OneMinMax and LeadingOnesTrailingZeroes), using an archive brings a polynomial acceleration on the expected running time. The reason is that with an archive, the size of the population can reduce to a small constant; there is no need for the population to keep all the Pareto optimal solutions found. This contrasts existing theoretical studies for MOEAs where a population with a commensurable size of the problem's Pareto front is needed. The findings in this paper not only provide a theoretical confirmation for an increasingly popular practice in the design of MOEAs, but can also be beneficial to the theory community towards studying more practical MOEAs. 
\end{abstract}

\section{Introduction}

Multi-objective optimization refers to an optimization scenario of considering multiple objectives simultaneously. Since the objectives of a multi-objective optimization problem (MOP) are usually conflicting, there does not exist a single optimal solution, but a set of trade-off solutions, called Pareto optimal solution set (or Pareto front in the objective space). Evolutionary algorithms (EAs)~\cite{back:96}, a kind of randomized heuristic optimization algorithms inspired by natural evolution, have been found well-suited to MOPs. Their population-based search mechanism can approximate a set of Pareto optimal solutions within one execution, with one solution representing a different trade-off between the objectives. Over the last decades, there have been a lot of well-known multi-objective EAs (MOEAs) developed, %~\cite{coello2007evolutionary}, 
such as the non-dominated sorting genetic algorithm II (NSGA-II)~\cite{deb-tec02-nsgaii}, multi-objective evolutionary algorithm based on decomposition (MOEA/D)~\cite{zhang2007moea}, and $\mathcal{S}$ metric selection evolutionary multi-objective optimization algorithm (SMS-EMOA)~\cite{beume2007sms}.

In the MOEA area, recently there is a trend of using an archive to store non-dominated solutions generated during the search~\cite{li2023multi}. A major reason for this practice is that the evolutionary population in MOEAs may fail to preserve high-quality solutions found even with elite preservation~\cite{knowles2004bounded}. During the evolutionary process, it is highly likely that there are more non-dominated solutions generated than the capacity of the population. As such, a population truncation needs to take place to remove excess non-dominated solutions (e.g., on the basis of their crowdedness in the population). However, the population later may accept new solutions which are non-dominated to the current population but being dominated by the ones removed previously (due to the dynamics of the evolution, for example, some area becomes sparse again). Consequently, MOEAs may end up with the final population having a large portion of (globally) dominated solutions. This applies to all practical MOEAs~\cite{li2023multi}. For example, it has been reported in~\cite{li2019empirical} that on some problems (such as DTLZ7~\cite{deb2005scalable}), nearly half of the final population of NSGA-II and MOEA/D are not globally optimal (i.e., being dominated by other solutions which were discarded during their evolutionary process).
This unwelcome issue has been frequently observed in various scenarios, from synthetic test suites~\cite{laumanns2002combining,fieldsend2003using,li2019empirical} to real-world problems~\cite{fieldsend2017university,chen2019standing}.  %reed2013evolutionary

\begin{table*}[t]
	\centering
	\begin{tabular}{l|l|l}
		\toprule
            & \omm & \lotz  \\ \midrule
            \makecell[l]{Original NSGA-II, namely, the one without\\ using an archive (\cite{bian2022better}) } & $O(\mu n\log n)$ [$\mu\ge 2(n+1)$] & $O(\mu n^2)$ [$\mu\ge 2(n+1)$] \\ \midrule
		 NSGA-II using an archive (this paper) & $O(\mu n\log n)$ [Thm~\ref{thm:nsga-arc-omm}, $\mu\ge 4$] & $O(\mu n^2+\mu^2 n\log n)$ [Thm~\ref{thm:nsga-arc-lotz}, $\mu\ge 4$] \\ \midrule
          \makecell[l]{Original SMS-EMOA, namely, the one without \\using an archive (\cite{zheng2024sms}) } & $O(\mu n\log n)$ [$\mu\ge n+1$] & $O(\mu n^2)$ [$\mu\ge n+1$] \\ \midrule
        SMS-EMOA using an archive (this paper) & $O(\mu n\log n)$ [Thm~\ref{thm:sms-arc-omm}, $\mu\ge 2$]  &  $O(\mu n^2+\mu^2 n\log n)$ [Thm~\ref{thm:sms-arc-lotz}, $\mu\ge 2$] \\ 
        \bottomrule
	\end{tabular} 
 \caption{The expected number of  fitness evaluations of NSGA-II and SMS-EMOA for solving the \omm\ and \lotz\ problems when an archive is used or not, where $n$ denotes the problem size, and $\mu$ denotes the population size.}
	\label{tab:summary}
\end{table*}

An easy way to fix the above issue is to use an (unbounded) archive that stores all non-dominated solutions found, leaving the search focused on finding new non-dominated solutions. This is indeed an increasingly popular practice~\cite{li2023multi}, given the capacity of today's computers that storing millions of solutions does not pose a problem. Various studies of using an unbounded archive emerged~\cite{li2023multi}, including to store high-quality solutions~\cite{dubois2015,ishibuchi2020,zhang2023dual}, to incorporate it into MOEAs as an important algorithm component~\cite{wang2019,li2021two}, to benchmark different MOEAs~\cite{tanabe2017,brockhoff2019}, to use it to identify if the search stagnates~\cite{li2023b}, and to design efficient data structure for it~\cite{glasmachers2017,jaszkiewicz2018,fieldsend2020,nan2020}. An archive (even a bounded one) can significantly improve the performance of MOEAs, as shown empirically in~\cite{bezerra2019}.

In this paper, we theoretically show that using an archive can bring speed-ups for MOEAs. Specifically, we study the expected running time of two well-established MOEAs, NSGA-II and SMS-EMOA, when using an archive to store non-dominated solutions generated during the search. We consider two bi-objective optimization problems, OneMinMax and LeadingOnes-TrailingZeroes, which were commonly used in theoretical studies of MOEAs~\cite{LaumannsTEC04,doerr2013lower,nguyen2015sibea,bian2022better,zheng2023first}.

The results are given in Table~\ref{tab:summary}. As can be seen in the table, using an archive allows a small constant population size, which brings an acceleration of $\Theta(n)$ on the expected running time. Note that the original algorithms without using an archive require the population size $\mu$ to be at least the same size of the problem’s Pareto front (i.e., $\mu \geq 2(n+1)$ for NSGA-II and $\mu \geq n+1$ for SMS-EMOA, where $n+1$ is the size of the Pareto front), otherwise Pareto optimal solutions can be lost even if they have been found previously. Using an archive that stores all the Pareto optimal solutions generated enables the algorithms not to worry about losing Pareto optimal solutions, but only endeavoring to seek new Pareto optimal solutions, thus speeding up the search process.

Over the past two decades, there is a substantial number of theoretical studies in the area of MOEAs, particularly regarding their running time complexity analyses. It starts with the analysis of a simple evolutionary multi-objective optimizer (SEMO) and its global variant, GSEMO, for solving multi-objective synthetic and combinatorial optimization problems~\cite{Giel03,LaumannsTEC04,Neumann07,Giel10,Neumann10,doerr2013lower,Qian13,bian2018tools}. On the other side, based on SEMO and GSEMO, the effectiveness of several parent selection and reproduction methods has also been studied~\cite{LaumannsTEC04,friedrich2011illustration,Qian13,qian-ppsn16-hyper,doerr2021ojzj}.

Recently, attention is being shifted to the analysis of practical MOEAs. The expected running time of $(\mu+1)$ SIBEA, i.e., an MOEA using the hypervolume indicator to update the population, was analyzed on several synthetic problems~\cite{brockhoff2008analyzing,nguyen2015sibea,doerr2016runtime}. Very recently, Zheng and Doerr~\shortcite{zheng2023first} analyzed the expected running time of NSGA-II, for the first time, by considering the bi-objective \omm\ and \lotz\ problems. Since then, the effectiveness of different components and mechanisms in NSGA-II, 
e.g., crowding distance~\cite{zheng2022current}, stochastic tournament selection~\cite{bian2022better}, fast mutation~\cite{doerr2023ojzj} and crossover~\cite{dang2023crossover,doerr2023crossover}, has also been analyzed.
More results about NSGA-II include \cite{cerf2023first,doerr2023ojzj,doerr2023lower,zheng2023manyobj}.
Furthermore, the expected running time of other well-established MOEAs has also been analyzed, e.g., MOEA/D~\cite{huang2019running,huang2020runtime,huang2021runtime}, SMS-EMOA~\cite{bian23stochastic,zheng2024sms}, and NSGA-III~\cite{wietheger23nsgaiii}, in addition to the analysis of them under different optimization models such as under noise~\cite{dang2023analysing,dinot2023runtime} and with the interactive model~\cite{lu2024imoea}.

Yet, all the above work regarding practical MOEAs needs a population with a commensurable size of the problem’s Pareto front. This may not be very practical since one may not be able to know the size of the problem's Pareto front before the optimization. The proposed work in this paper addresses this issue and proves that a small population, with an archive, even works better. This result not only provides a theoretical confirmation for an increasingly popular practice in the development of MOEAs, but can also be beneficial to the theory community towards studying more practical MOEAs.

% \begin{table*}[t]
% 	\centering
% 	\caption{The expected number of  fitness evaluations of NSGA-II and SMS-EMOA for solving the \omm\ and \lotz\ problems when an archive is used, where $n$ denotes the problem size, and $\mu$ denotes the population size.}
% 	\label{tab:summary}
% 	\begin{tabular}{l|l|l}
% 		\toprule
%             & \omm & \makecell[l]{LeadingOnes-\\TrailingZeroes}  \\ \midrule
% 		 \makecell[l]{NSGA-II \\using an archive}  & \makecell[l]{$O(\mu n\log n)$\\ Thm $\mu\ge 4$} & \makecell[l]{$O(\mu n^2+\mu^2 n\log n)$\\ $\mu\ge 4$} \\ \midrule
%         \makecell[l]{SMS-EMOA  \\using an archive}  & \makecell[l]{$O(\mu n\log n)$\\ $\mu\ge 2$ }    &  \makecell[l]{$O(\mu n^2+\mu^2 n\log n)$ \\$\mu\ge 2$ }\\ \bottomrule
% 	\end{tabular} 
% \end{table*}

\section{Preliminaries}\label{sec-preliminary}

In this section, we first give basic concepts in multi-objective optimization, which is followed by the considered algorithms NSGA-II and SMS-EMOA, and the archive mechanism. Lastly, we describe the \omm\ and \lotz\ problems studied in this paper.

\subsection{Multi-objective Optimization}

Multi-objective optimization aims to optimize two or more objective functions simultaneously, as presented in Definition~\ref{def_MO}. In this paper, we consider maximization (minimization can be defined similarly), and pseudo-Boolean functions, i.e., the solution space $\mathcal{X}=\{0,1\}^n$. Since the objectives are usually conflicting, there does not exist canonical complete order in the solution space $\mathcal{X}$, and we use the \emph{domination} relationship in Definition~\ref{def_Domination} to compare solutions. A solution is \emph{Pareto optimal} if it is not dominated by any other solution in $\mathcal{X}$, and the set of objective vectors of all the Pareto optimal solutions is called the \emph{Pareto front}. The goal of multi-objective optimization is to find the Pareto front or its good approximation.

\begin{definition}[Multi-objective Optimization]\label{def_MO}
	Given a feasible solution space $\mathcal{X}$ and objective functions $f_1,f_2,\ldots, f_m$, multi-objective optimization can be formulated as
	\[
	\max_{\bmx\in
		\mathcal{X}}\bmf(\bmx)=\max_{\bmx \in
		\mathcal{X}} \big(f_1(\bmx),f_2(\bmx),...,f_m(\bmx)\big).
	\]
\end{definition}
\begin{definition}[Domination]\label{def_Domination}
	Let $\bm f = (f_1,f_2,\ldots, f_m):\mathcal{X} \rightarrow \mathbb{R}^m$ be the objective vector. For two solutions $\bmx$ and $\bmy\in \mathcal{X}$:
	\begin{itemize}
		\item $\bmx$ \emph{weakly dominates} $\bmy$  (denoted as $\bmx \succeq \bmy$) if for any $1 \leq i \leq m, f_i(\bmx) \geq f_i(\bmy)$;
		\item $\bmx$ \emph{dominates} $\bmy$ (denoted as $\bmx\succ \bmy$) if $\bm{x} \succeq \bmy$ and $f_i(\bmx) > f_i(\bmy)$ for some $i$;
		\item  $\bmx$ and $\bmy$ are \emph{incomparable} if neither $\bmx\succeq \bmy$ nor $\bmy\succeq \bmx$.
	\end{itemize}
\end{definition}

\subsection{NSGA-II and SMS-EMOA }

\begin{algorithm}[tb]
	\caption{NSGA-II}
	\label{alg:nsgaii}
	\textbf{Input}: objective functions $f_1,f_2\ldots,f_m$, population size $\mu$, probability $p_c$ of using crossover
    \begin{algorithmic}[1]
		\STATE $P\leftarrow \mu$ solutions uniformly and randomly selected from $\{0,\! 1\}^{\!n}$ with replacement;
		\WHILE{criterion is not met}
		\STATE let $P'=\emptyset, i=0$;  
		\WHILE{$i<\mu/2$}
		\STATE apply binary tournament selection twice to select two solutions $\bmx$ and $\bmy$;
		\STATE sample $u$ from uniform distribution over $[0,1]$;
		\IF{$u<p_c$}
		\STATE apply one-point crossover on $\bmx$ and $\bmy$ to generate two solutions $\bmx'$ and $\bmy'$
		\ELSE 
		\STATE set $\bmx'$ and $\bmy'$ as copies of $\bmx$ and $\bmy$, respectively
		\ENDIF
		\STATE apply bit-wise mutation on $\bmx'$ and $\bmy'$ to generate $\bmx''$ and $\bmy''$, respectively, and add them into $P'$;
		\STATE $i=i+1$
		\ENDWHILE
		\STATE partition $P\cup P'$ into non-dominated sets $R_1,\ldots, R_v$;
		\STATE let $P=\emptyset$, $i=1$;
		\WHILE{$|P\cup R_i|<\mu$}
		\STATE $P=P\cup R_i$, $i=i+1$
		\ENDWHILE
		\STATE  assign each solution in $R_i$ with a crowding distance; 
		\STATE  sort the solutions in $R_i$ in ascending order by crowding distance, and add the last $\mu-|P|$ solutions into $P$
		\ENDWHILE
		\RETURN $P$
	\end{algorithmic}
\end{algorithm}

The NSGA-II algorithm~\cite{deb-tec02-nsgaii}, as presented in Algorithm~\ref{alg:nsgaii}, is a very popular MOEA which incorporates two substantial features, i.e., non-dominated sorting and crowding distance. It starts from an initial population of $\mu$ (without loss of generality, we assume that $\mu$ is even) random solutions (line~1). 
In each generation, NSGA-II employs binary tournament selection to select parent solutions (line~5), which picks two solutions randomly from the population $P$ with replacement, and then selects a better one as the parent solution (ties broken uniformly).
Then, one-point crossover is performed on the two parent solutions with probability $p_c$ (lines~6--11), which selects a  crossover point $i\in \{1,2,\ldots,n\}$ uniformly at random,  where $n$ is the problem size, and then exchanges the first $i$ bits of two solutions. The bit-wise mutation operator, which flips each bit of a solution independently with probability $1/n$, is then applied to generate offspring solutions (line~12).
After a set $P'$ of $\mu$ offspring solutions have been generated, the solutions in the current and offspring populations are partitioned into non-dominated sets $R_1,\ldots,R_v$ (line~15), where $R_1$ contains all the non-dominated solutions in $P\cup P'$, and $R_i$ ($i\ge 2$) contains all the non-dominated solutions in $(P\cup P') \setminus \cup_{j=1}^{i-1} R_j$.  Note that  a solution is said to be with rank $i$ if it belongs to $R_i$. Then, the solutions in $R_1,\ldots, R_v$ are added into the next population, until the population size exceeds $\mu$ (lines~16--19). For the critical set $R_i$ whose inclusion makes the population size larger than $\mu$, the crowding distance is computed for each of the contained solutions (line~20).
Crowding distance reflects the diversity of a solution. For each objective $f_j $, $1\le j\le m$, the solutions in $R_i$ are sorted according to their objective values in ascending order, and we assume the sorted list is $\bmx^1,\bmx^2,\ldots,\bmx^k$; the crowding distance of the solution $\bmx^l$ with respect to $f_j$ is set to $\infty$ if $l\in \{1,k\}$, and $(f_j(\bmx^{l+1})-f_j(\bmx^{l-1}))/(f_j(\bmx^k)-f_j(\bmx^1))$ otherwise.
The final crowding distance of a solution is the sum of the crowding distance with respect to each objective. 
Finally, the solutions in $R_i$ are selected to fill the remaining population slots where the solutions with larger crowding distance are preferred (line~21). 
Note that when using binary tournament selection in line~5, the selection criterion is based on rank and crowding distance, that is, a solution $\bmx$ is superior to $\bmy$ if $\bmx$ has a smaller rank, or $\bmx$ and $\bmy$ have the same rank but $\bmx$ has a larger crowding distance than $\bmy$. The probability $p_c$ of using crossover is set  to $0.9$, just as the  setting in~\cite{deb-tec02-nsgaii}.

\begin{algorithm}[tb]
	\caption{SMS-EMOA}
	\label{alg:sms-emoa}
	\textbf{Input}: objective functions $f_1,f_2\ldots,f_m$, population size $\mu$, probability $p_c$ of using crossover 
	%	\textbf{Parameter}: Optional list of parameters\\
	\begin{algorithmic}[1] %[1] enables line numbers
		\STATE $P\leftarrow \mu$ solutions uniformly and randomly selected from $\{0,\! 1\}^{\!n}$ with replacement;
		\WHILE{criterion is not met}
		\STATE select a  solution $\bmx$ from $P$ uniformly at random;
		\STATE sample $u$ from uniform distribution over $[0,1]$;
		\IF{$u<p_c$}
		\STATE select a  solution $\bmy$ from $P$ uniformly at random;
		\STATE apply one-point crossover on $\bmx$ and $\bmy$ to generate one solution $\bmx'$
		\ELSE 
		\STATE set $\bmx'$ as the copy of $\bmx$
		\ENDIF
		\STATE apply bit-wise mutation on $\bmx'$ to generate $\bmx''$;
		\STATE partition $P\!\cup \!\{\bmx''\!\}$ into non-dominated sets $R_1\!,...,\!R_v$;
		\STATE let $\bmz=\arg\min_{\bmx\in R_v}\Delta_{\bmr}(\bmx,R_v)$;
		\STATE $P\leftarrow (P\cup \{\bmx''\})\setminus \{\bmz\}$
		\ENDWHILE
		\RETURN $P$
	\end{algorithmic}
\end{algorithm}

The SMS-EMOA algorithm~\cite{beume2007sms} as presented in Algorithm~\ref{alg:sms-emoa} is also a popular MOEA, which  employs non-dominated sorting and hypervolume indicator to update the population.
Starting from an initial population of $\mu$ random solutions (line~1), in each generation, it  randomly selects a parent solution $\bmx$ from the current population for reproduction (line~3). With probability $p_c$ (similar to NSGA-II, $p_c$ is set to $0.9$), it selects another solution $\bmy$ and applies one-point crossover on $\bmx$ and $\bmy$ to generate an offspring solution $\bmx'$ (lines~4--7); otherwise, $\bmx'$ is set as the  copy of $\bmx$ (line~9). Note that one-point crossover actually produces two solutions, but the algorithm only picks the one that consists of the first part of the first parent solution and the second part of the second parent solution. Afterwards, bit-wise mutation is applied on $\bmx'$ to generate one offspring solution (line~11). 
Then, similar to line~15 of Algorithm~\ref{alg:nsgaii}, the union of the current population and the newly generated offspring solution is partitioned into non-dominated sets $R_1,\ldots,R_v$ (line~12), and one solution $\bmz\in R_v$ that minimizes
$
\Delta_{\bmr}(\bmx, R_v):=HV_{\bmr}(R_v)-HV_{\bmr}(R_v\setminus \{\bmx\})
$
is removed (lines~13--14), where $HV_{\bmr}(X)
	=\Lambda
	\big(\cup_{\bmx\in X} 
	\{\bmf'\in \mathbb{R}^m \mid 
	\forall 1\le i\le m: 
	r_i\le f'_i\le f_i(\bmx)\}\big) $
denotes the hypervolume of a solution set $X$ with respect to a reference point $\bmr\in \mathbb{R}^m$ (satisfying $\forall 1\le i\le m, r_i\le \min_{\bmx\in \mathcal{X}}f_i(\bmx)$), i.e., the volume of the objective space between the reference point and the objective vectors of the solution  set, and $\Lambda$ denotes the Lebesgue measure. 
A larger hypervolume value implies a better approximation with regards to both convergence and diversity. 
Note that when SMS-EMOA solves bi-objective problems, we use the original setting in~\cite{beume2007sms} that the two extreme points (i.e., the objective vectors which contain the largest objective value for some $f_i$, $i\in \{1,2\}$) are always kept in the population regardless of their hypervolume loss.
 % i.e., the hypervolume of the two extreme points can be viewed as infinity.

\subsection{NSGA-II and SMS-EMOA with an Archive}
In the original NSGA-II and SMS-EMOA, there is no archive used. As explained previously, the non-dominated solutions may lose even if they have been found once. Using an archive can easily address this issue. That is, once a new solution is generated, the solution will be tested if it can enter the archive. If there is no solution in the archive that dominates the new solution, then the solution will be placed in the archive. Additional algorithmic steps incurred by adding an archive in NSGA-II and SMS-EMOA are given as follows. 
For NSGA-II in Algorithm~\ref{alg:nsgaii}, an empty set $Q$ is  initialized in line~1, the following lines
\begin{framed}\vspace{-0.8em}
    \begin{algorithmic}
    \FOR{$\bmx''\in P'$}
    \IF{$\nexists \bmz \in Q$ such that $\bmz \succ \bmx''$}
    \STATE $Q \leftarrow (Q \setminus\{\bmz \in Q \mid \bmx'' \succeq \bmz\}) \cup \{\bmx''\}$
    \ENDIF
    \ENDFOR
\end{algorithmic}\vspace{-0.8em}
\end{framed}\noindent
are added after line~14, and the set $Q$ instead of $P$ is returned in the last line. For SMS-EMOA in Algorithm~\ref{alg:sms-emoa}, an empty set $Q$ is also initialized in line~1, 
the following lines
\begin{framed}\vspace{-0.8em}
    \begin{algorithmic}
    \IF{$\nexists \bmz \in Q$ such that $\bmz \succ \bmx''$}
    \STATE $Q \leftarrow (Q \setminus\{\bmz \in Q \mid \bmx'' \succeq \bmz\}) \cup \{\bmx''\}$
    \ENDIF
\end{algorithmic}\vspace{-0.8em}
\end{framed}\noindent
are added after line~11, and the set $Q$ instead of $P$ is returned in the last line. 

\subsection{\omm\ and \lotz}

Now we introduce two bi-objective problems, \omm\ and \lotz\ considered in this paper. Theses two problems have been widely used in MOEAs’ theoretical analyses~\cite{LaumannsTEC04,brockhoff2008analyzing,doerr2013lower,nguyen2015sibea,bian2022better,zheng2023first}. 

The \omm\ problem presented in Definition~\ref{def:OMM} aims to simultaneously maximize the number of 0-bits and the number of 1-bits of a binary bit string.  
The Pareto front is $\{(a, n-a)\mid a\in [0..n]\}$, whose size is $n+1$, and 
the Pareto optimal solution corresponding to $(a, n-a)$, $a\in [0..n]$, is any solution with $(n-a)$ 1-bits. Note that we use $[l..r]$ to denote the set  $\{l,l+1,\ldots, r\}$ of integers throughout the paper. We can see that any solution $\bmx\in\{0,1\}^n$ is Pareto optimal for this problem.
\begin{definition}[\omm~\cite{Giel10}]\label{def:OMM}
	The OneMinMax problem of size $n$ is to find $n$ bits binary strings which maximize
        $
		{\bm f}(\bmx)=\left(n-\sum\nolimits^n_{i=1}x_i, \sum\nolimits^{n}_{i=1} x_i\right)
        $,
	where $x_i$ denotes the $i$-th bit of $\bmx \in \{0,1\}^n$.
\end{definition}

The \lotz\  problem presented in Definition~\ref{def:LOTZ} aims to simultaneously maximize the number of leading 1-bits and the number of trailing 0-bits of a binary bit string.  
The Pareto front is $\{(a, n-a)\mid a\in [0..n]\}$, whose size is $n+1$, and the Pareto optimal solution corresponding to $(a,n-a)$,  $a\in [0..n]$, is $1^a0^{n-a}$, i.e., the solution with $a$ leading 1-bits and $n-a$ trailing 0-bits.
\begin{definition}[\lotz~\cite{LaumannsTEC04}]\label{def:LOTZ}
	The \lotz\ problem of size $n$ is to find $n$ bits binary strings which maximize
        $
        	{\bm{f}}(\bmx)= (\sum\nolimits^n_{i=1} \prod\nolimits^{i}_{j=1}x_j, \sum\nolimits^{n}_{i=1} \prod\nolimits^{n}_{j=i}(1-x_j))
         $,
	where $x_j$ denotes the $j$-th bit of $\bmx \in \{0,1\}^n$.
\end{definition}

\section{Analysis of NSGA-II with an Archive}

In this section, we analyze the expected running time of NSGA-II in Algorithm~\ref{alg:nsgaii} using an archive.
Note that the running time of an EA is usually measured by the number of fitness evaluations, which is often the most time-consuming step in the evolutionary process.
We prove in Theorem~\ref{thm:nsga-arc-omm} that the expected number of fitness evaluations of NSGA-II using an archive for solving \omm\ is $O(\mu n\log n)$, where the population size $\mu \ge 4$. The proof idea is to divide the optimization procedure into two phases, where the first phase aims at finding the two extreme Pareto optimal solutions $1^n$ and $0^n$, and the second phase aims at finding the remaining objective vectors in the Pareto front.

\begin{theorem}\label{thm:nsga-arc-omm}
	For NSGA-II solving \omm, if using an archive, and a population size $\mu$ such that $\mu\ge 4$, then the expected number of fitness evaluations for finding the Pareto front is $O(\mu  n\log n)$.
\end{theorem}
\begin{proof}
    We divide the running process of NSGA-II into two phases. The first phase starts after initialization and finishes until $1^n$ and $0^n$ are both found; the second phase starts after the first phase and finishes when the Pareto front is found. 
	
    For the first phase, we will prove that the expected number of generations for finding $0^n$ is $O(n\log n)$, and then the same bound holds for finding $1^n$ analogously. We first show that the maximal $f_1$ value of the solutions in the population $P$, i.e., $\max_{\bmx\in P} |\bmx|_0$, will not decrease, where $|\bmx|_0$ denotes the number of 0-bits in $\bmx$. Let $C$ denote the set of solutions in $P\cup P'$ with the maximal $f_1$ value, where $P'$ denotes the offspring population. Because of $P \subseteq P \cup P'$ and the definition of $C$, we have for any $\bm{x} \in C$, $|\bmx|_0\ge \max_{\bmx\in P} |\bmx|_0$. Thus, we only need to show that one solution in $C$ will be maintained in the next population. Because all the solutions in $C$ have the maximal $f_1$ value and the same $f_2$ value, they cannot be dominated by any solution in $P\cup P'$, implying that they all have rank $1$, i.e., belong to $R_1$ in the non-dominated sorting procedure. If $|R_1|\le \mu$, all the solutions in $C$ will be maintained in the next population. If $|R_1|> \mu$, the crowding distance of the solutions in $R_1$ needs to be computed. When the solutions in $R_1$ are sorted according to $f_1$ in ascending order, one solution in $C$ (denoted as $\bmx^*$) must be put in the last position and thus has infinite crowding distance. Note that only solutions in the first and the last positions can have infinite crowding distance. As \omm\ has two objectives, at most four solutions in $R_1$ can have infinite crowding distance. Thus, $\bmx^*$ is among the best four solutions in $R_1$  and must be included in the next population (note that $\mu\ge 4$), implying that the maximal $f_1$ value will not decrease.
	
    Next, we analyze the increase of the maximal $f_1$ value. Assume that currently the maximal $f_1$ value is $i$ ($i<n$), i.e., $\max_{\bmx\in P} |\bmx|_0=i$.	When using binary tournament selection to select a parent solution, the competition between the two randomly selected solutions is based on rank and crowding distance (note that the rank and crowding distance here are computed based on the current population $P$ instead of the union of the current population and the offspring population) with ties broken uniformly. As analyzed in the last paragraph, a solution $\bmx\in P$ with $|\bmx|_0=i$ must have rank $1$ and infinite crowding distance. Once $\bmx$ is selected for competition (whose probability is $1/\mu$), it will always win, if the other solution selected for competition has larger rank or finite crowding distance; or win with probability $1/2$, if the other solution has the same rank and crowding distance as $\bmx$, resulting in a tie which is broken uniformly at random. Thus, $\bmx$ can be selected as a parent solution with probability at least $1/(2\mu)$.
    In the reproduction procedure, a solution with more 0-bits can be generated from $\bmx$ if crossover is not performed (whose probability is $1-0.9=0.1$) and only one of the 1-bits in $\bmx$ is flipped by bit-wise mutation (whose probability is $((n-|\bmx|_0)/n)\cdot (1-1/n)^{n-1}$). Thus, the probability of generating a solution with more than $i$ 0-bits is at least 
	\begin{equation}\label{eq:selection}
	 \frac{1}{2\mu}\cdot 0.1\cdot \frac{n-|\bmx|_0}{n}\cdot \Big(1-\frac{1}{n}\Big)^{n-1}\ge \frac{n-i}{20e\mu n}.
 \end{equation}
	Because in each generation, $\mu/2$ pairs of parent solutions will be selected for reproduction,  the probability of generating a solution with more than $i$ 0-bits is at least 
	\begin{equation}\label{eq:nsga-omm-phase1-2}
		\begin{aligned}
			&1-\Big(1-\frac{n-i}{20e\mu n}\Big)^{\mu/2}
			\ge  1-\frac{1}{e^{(n-i)/(40en)}}\\
			&\ge  1-\frac{1}{1+(n-i)/(40en)}= \Omega\Big(\frac{n-i}{n}\Big),
			%\ge 1-e^{-N/(4enN)}\ge 1-e^{-1/(2e)} =\Omega(1),
		\end{aligned}
	\end{equation}
	where the inequalities hold by $1+a\le e^a$ for any $a\in \mathbb{R}$. 	
	Because the solution with the most number of 0-bits will be maintained in the population, the expected number of generations for increasing the maximal $f_1$ value to $n$, i.e., finding $0^n$, is at most 
%	\[
%	\sum_{i=0}^{n-1}O\Big(\frac{n}{n-i}\Big)=O(n \log n).
%	\]
	$\sum_{i=0}^{n-1}O(n/(n-i))=O(n \log n)$.
	That is, the expected number of generations of the first phase is $O(n\log n)$. 
	
	Now we consider the second phase, and will show that NSGA-II can find the whole Pareto front in $O(n\log n)$ expected numbers of generations. Note that after phase~1, $0^n$ and $1^n$ must be maintained in the population $P$. Let $D=\{j \mid \exists \bmx\in Q, |\bmx|_0=j\}$, where $Q$ denotes the archive, and we suppose $|D|=i$, i.e., $i$ points in the Pareto front has been found in the archive. Note that $i \geq 2$ as $0^n$ and $1^n$ have been found.  Next, we consider two cases.\\
	(1) The number of $1^n$ or the number of $0^n$ in the current population $P$ is at least $\mu/4$.  Without loss of generality, we assume that the number of $0^n$ is  at least $\mu/4$. Then,  the probability of selecting $0^n$ as a parent solution is at least $(1/4)^2=1/16$,  because it is sufficient to select $0^n$ twice in binary tournament selection. According to the analysis in the  paragraph above Eq.~\eqref{eq:selection},  the probability of selecting $1^n$ as the other parent solution is at least $1/(2\mu)$. After exchanging the first $k$ ($k\in [0..n]\setminus D$) bits of $0^n$ and $1^n$ by one-point crossover (the probability is $0.9 \cdot (1/n)$), a solution with $k$ 0-bits can be generated, which can keep unchanged after bit-wise mutation if none of the bits is flipped (the probability is $(1-1/n)^{n}$). Thus, the probability of generating a new point in the Pareto front is at least 
    \begin{equation}\label{eq:nsga-omm-phase2-1}
	\begin{aligned}
		&\frac{1}{16}\cdot \frac{1}{2\mu}\cdot 0.9\cdot \frac{n+1-i}{n} \cdot \Big(1-\frac{1}{n}\Big)^n\\
		&\ge \frac{n+1-i}{32\mu n}\cdot 0.9\cdot \Big(1-\frac{1}{n}\Big)\cdot \frac{1}{e}\ge \frac{n+1-i}{64e\mu n},
	\end{aligned}
    \end{equation}
    where the term $n+1-i$ is because there are $|[0..n]\setminus D|=n+1-i$ points in the Pareto front to be found, the first inequality holds by $(1-1/n)^{n-1}\ge 1/e$, and the second inequality holds for $n\ge 3$. \\
    (2) The number of $1^n$ and the number of $0^n$ in the current population $P$ are both less than $\mu/4$. Then, in one binary tournament selection procedure, the probability of selecting two solutions with the number of 0-bits in $[1..n-1]$ is at least $(\mu-\mu/4-\mu/4)^2/\mu^2=1/4$, and we assume that the winning solution $\bmx$ has $j$ ($j\in [1..n-1]$) 0-bits. 
    If the other selected parent solution is $0^n$, then for any $k\in [j+1..n]\setminus D$,  there must exist a crossover point $k'$ such that exchanging the first $k'$ bits of $\bmx$ and $0^n$ can generate a  solution with $k$ 0-bits. 
    If the other selected parent solution is $1^n$, then for any $k\in [0..j-1]\setminus D$,  there must exist a crossover point $k'$ such that exchanging the first $k'$ bits of $\bmx$ and $1^n$ can generate a solution with $k$ 0-bits. The newly generated solution can keep unchanged  by flipping no bits in bit-wise mutation. Note that the probability of selecting $1^n$ (or $0^n$) as a parent solution is at least $1/(2\mu)$. Thus, similar to Eq.~\eqref{eq:nsga-omm-phase2-1},
    the probability of generating  a new point in the Pareto front is at least
    \begin{equation}\label{eq:nsga-omm-phase2-2}
    	\begin{aligned}
    		\frac{1}{4}\cdot \frac{1}{2\mu}\cdot 0.9\cdot \frac{n+1-|D|}{n}\cdot \Big(1-\frac{1}{n}\Big)^n
    		\ge \frac{n+1-i}{16e\mu n}.
    	\end{aligned}
    \end{equation}
    By taking the smaller one between Eqs.~\eqref{eq:nsga-omm-phase2-1} and~\eqref{eq:nsga-omm-phase2-2}, and using the fact that NSGA-II produces $\mu/2$  pairs of offspring solutions in each generation, the probability of generating a new point in the Pareto front in each generation is at least 
	\begin{equation}\label{eq:nsga-omm-phase2-3}
		\begin{aligned}
			1-\Big(1-\frac{n+1-i}{64e\mu n}\Big)^{\mu/2}=
			\Omega\Big(\frac{n+1-i}{n}\Big).
		\end{aligned}
	\end{equation}
	Then, we can derive that the expected number of generations of the second phase (i.e., for finding the whole Pareto front) is at most 
%	\[
%	\sum_{i=2}^{n}\frac{128e n+n+1-i}{n+1-i}=O(n \log n).
%	\] 
	$\sum_{i=2}^{n}O(n/(n+1-i))=O(n \log n)$.
	
	Combining the two phases, the total expected number of generations is  $O(n\log n)$, implying that the expected number of fitness evaluations is $O(\mu  n\log n)$, because each generation of NSGA-II  requires to evaluate $\mu$ offspring solutions. Thus, the theorem holds.	
\end{proof}

We prove in Theorem~\ref{thm:nsga-arc-lotz} that the expected number of fitness evaluations of NSGA-II using an archive for solving \lotz\ is $O(\mu n^2+\mu^2 n\log n)$, where the population size $\mu \ge 4$.
As the proof of Theorem~\ref{thm:nsga-arc-omm}, we divide the optimization procedure into two phases, that is to find the extreme solutions $1^n$ and $0^n$, and to find the whole Pareto front, respectively.  
But due to the fact that only solutions with the form $1^j0^{n-j}$ are Pareto optimal for \lotz\ while any solution is Pareto optimal for \omm, their analyses of increasing the maximal $f_1$ value in the first phase as well as generating new Pareto optimal solutions in the second phase are different. 
In the following proof, we use $\lo(\bmx)=\sum_{i=1}^n\prod_{j=1}^i x_j$ and $\tz(\bmx)=\sum_{i=1}^n\prod_{j=i}^n(1-x_j)$ to denote the number of leading 1-bits and trailing 0-bits of a solution $\bmx$, respectively.
\begin{theorem}\label{thm:nsga-arc-lotz}
	For NSGA-II solving \lotz, if using an archive, and a population size $\mu$ such that $\mu\ge 4$, then the expected number of fitness evaluations for finding the Pareto front is $O(\mu n^2+\mu^2 n\log n)$.
\end{theorem}
\begin{proof}
	The proof of the first phase is similar to that of Theorem~\ref{thm:nsga-arc-omm}. However, to show that the maximal $f_1$ value, i.e., $\max_{\bmx\in P} \lo(\bmx)$, will not decrease, we need to define $C$ as the set of solutions in $P\cup P'$ with the maximal number of  leading 1-bits, and further define a set $C^*=\{\bmx\in C\mid  \tz(\bmx)=\max_{\bmx\in C}\tz(\bmx)\}$, which denotes the set of solutions in $C$ with the maximal number of trailing 0-bits. As the analysis in the proof of Theorem~\ref{thm:nsga-arc-omm}, one solution in $C^*$ will be maintained in the next population, implying that the maximal $f_1$ value will not decrease.
	To make the maximal $f_1$ value increase, it is sufficient to select a solution $\bmx\in C^*$ as the parent solution and flip its $(\lo(\bmx)+1)$-th bit (which must be 0). Similar to the analysis of Eq.~\eqref{eq:selection}, the probability of  generating a solution with more leading 1-bits is at least 
	\begin{equation}\label{eq:nsga-lotz-phase1-1}
		\frac{1}{2\mu}\cdot 0.1\cdot \frac{1}{n}\cdot \Big(1-\frac{1}{n}\Big)^{n-1}\ge \frac{1}{20e\mu n}.
	\end{equation}
	Following the analysis of Eq.~\eqref{eq:nsga-omm-phase1-2}, the probability of generating a solution with more leading 1-bits in each generation is at least
		$	1-(1-1/(20e\mu n))^{\mu/2}= \Omega(1/n) $.
	As the maximal $f_1$ value can increase at most $n$ times, the expected number of generations of the first phase is $O(n^2)$. 
	
    Next we consider the second phase.	Let $D=\{j \in [0 .. n]\mid 1^j0^{n-j}\in Q\}$, where $Q$ denotes the archive, and suppose $|D|=i$, i.e., $i$ points in the Pareto front has been found in the archive. By selecting $1^n$ and $0^n$ as a pair of parent solutions, and exchanging their first $k$ bits ($k\in [0..n]\setminus D$) by one-point crossover, the solution $1^k0^{n-k}$ can be generated. $1^k0^{n-k}$ can keep unchanged by flipping no bits in bit-wise mutation.
	 As the analysis in the paragraph above Eq.~\eqref{eq:selection}, the probability of selecting $1^n$ (or $0^n$) as the parent solution by binary tournament selection is at least $1/(2\mu)$.
    Thus, the probability of generating a new Pareto optimal solution is at least 
	\begin{equation}\label{eq:nsga-lotz-phase2-1}
		\begin{aligned}
			\frac{1}{(2\mu)^2}\cdot 0.9\cdot \frac{n+1-i}{n} \cdot \Big(1-\frac{1}{n}\Big)^n \ge \frac{n+1-i}{8e\mu^2n},
		\end{aligned}
	\end{equation}
	where the inequality holds by large enough $n$.
	Following the analysis of Eq.~\eqref{eq:nsga-omm-phase2-3}, the probability of generating a new point in the Pareto front in each generation is at least
	\begin{equation}
			1-\Big(1-\frac{n+1-i}{8e\mu^2n}\Big)^{\mu/2}=\Omega\Big(\frac{n+1-i}{\mu n}\Big).
	\end{equation}
	Thus, the expected number of generations for finding the whole Pareto front is at most $\sum_{i=2}^{n}O(\mu n/(n+1-i))=O(\mu n \log n)$. 
	
	Combining the two phases, the total expected number of generations is  $O(n^2+\mu n\log n)$, implying that the expected number of fitness evaluations is $O(\mu n^2+\mu^2 n\log n)$. Thus, the theorem holds.	
\end{proof}

The expected number of fitness evaluations of the original NSGA-II (without using an archive) for solving \omm\ and \lotz\ has been shown to be $O(\mu n\log n)$ and $O(\mu n^2)$, respectively, where the population size $\mu\ge 2(n+1)$~\cite{bian2022better}. Thus, our results in Theorems~\ref{thm:nsga-arc-omm} and~\ref{thm:nsga-arc-lotz} show that if a constant population size is used for NSGA-II having an archive, the expected running time can be reduced by a factor of $\Theta(n)$. 
The main reason for the acceleration is that the archive can preserve all the non-dominated solutions generated so far, which enables NSGA-II to discard solutions that are not critical for finding the Pareto front, and thus to use a small population bringing more efficient exploration of the search space.

\section{Analysis of SMS-EMOA with an Archive}

In this section, we consider SMS-EMOA in Algorithm~\ref{alg:sms-emoa} using an archive.
We prove in Theorems~\ref{thm:sms-arc-omm} and~\ref{thm:sms-arc-lotz} that the expected number of fitness evaluations of SMS-EMOA using an archive for solving \omm\ and \lotz\ is $O(\mu n\log n)$ and $O(\mu n^2+\mu^2n\log n)$, respectively, where the population size $\mu \ge 2$.
Their proofs are similar to that of Theorems~\ref{thm:nsga-arc-omm} and~\ref{thm:nsga-arc-lotz}, respectively. That is, we divide the optimization procedure into two phases, where the first phase aims at finding  $1^n$ and $0^n$, and the second phase aims at finding the whole Pareto front. The main difference of the proofs is led by that 1) during the population update procedure, SMS-EMOA directly preserves two boundary objective vectors which contain the largest objective value for $f_1$ or $f_2$, and 2) during the reproduction procedure, it uses uniform parent selection and generates only one offspring solution in each generation. 
\begin{theorem}\label{thm:sms-arc-omm}
	For SMS-EMOA solving \omm, if using an archive, and a population size $\mu$ such that $\mu\ge 2$, then the expected number of fitness evaluations for finding the Pareto front is $O(\mu n \log n)$.
\end{theorem}

\begin{proof}
	For the first phase, the maximal $f_1$ value will not decrease because SMS-EMOA directly keeps the two boundary points in the population update procedure.	Now, we consider the increase of the maximal $f_1$ value. Note that SMS-EMOA in Algorithm~\ref{alg:sms-emoa} selects a parent solution uniformly at random, instead of using binary tournament selection; thus, the probability of selecting any specific solution in the population is $1/\mu$. Eq.~\eqref{eq:selection} changes to 
		\begin{equation}\label{eq:sms-omm-phase1-1}
		\frac{1}{\mu}\cdot 0.1\cdot \frac{n-|\bmx|_0}{n}\cdot \Big(1-\frac{1}{n}\Big)^{n-1}\ge \frac{n-i}{10e\mu n}.
		\end{equation}
		Different from NSGA-II which produces $\mu/2$ pairs of offspring solutions in each generation,  SMS-EMOA only reproduces one solution in each generation. Thus, the expected number of generations for increasing the maximal $f_1$ value to $n$, i.e., finding $0^n$, is at most 
%		\[
%		\sum_{i=0}^{n-1}\frac{10\mu en}{n-i}=O(\mu n \log n).
%		\]
		$	\sum_{i=0}^{n-1}10e\mu n/(n-i)=O(\mu n \log n)$.
		That is, the expected number of generations of the first phase is $O(\mu n\log n)$. 
		
		For the second phase, by considering the difference between uniform parent selection employed by SMS-EMOA and binary tournament selection employed by NSGA-II, Eq.~\eqref{eq:nsga-omm-phase2-1} changes to 
		\[
		\begin{aligned}
			\frac{1}{4}\cdot \frac{1}{\mu}\cdot 0.9\cdot \frac{n+1-i}{n} \cdot \Big(1-\frac{1}{n}\Big)^n 
			\ge \frac{n+1-i}{8e\mu n},
		\end{aligned}
		\]
		and Eq.~\eqref{eq:nsga-omm-phase2-2} changes to
		\[
			\begin{aligned}
				\frac{1}{2}\cdot \frac{1}{\mu}\cdot 0.9\cdot \frac{n+1-|D|}{n}\cdot \Big(1-\frac{1}{n}\Big)^n
				\ge \frac{n+1-i}{4e\mu n}.
			\end{aligned}
		\]
		Thus, the expected number of generations of the second phase (i.e., for finding the whole Pareto front) is at most 
%		\[
%		\sum_{i=2}^{n}\frac{8e\mu n}{n+1-i}=O(n \log n).
%		\] 
		$\sum_{i=2}^{n}8e\mu n/(n+1-i)=O(\mu n \log n)$.
		
		Combining the two phases, the total expected number of generations is  $O(\mu n\log n)$. Since SMS-EMOA only generates one solution in generation, the expected number of fitness evaluations is also $O(\mu  n\log n)$.
\end{proof}

\begin{theorem}\label{thm:sms-arc-lotz}
	For SMS-EMOA solving \lotz, if using an archive, and a population size $\mu$ such that $\mu\ge 2$, then the expected number of fitness evaluations for finding the Pareto front is $O(\mu n^2+\mu^2 n\log n)$.
\end{theorem}
\begin{proof}
    By following the proof of Theorem~\ref{thm:nsga-arc-lotz} and considering the difference between the parent selection strategies employed by NSGA-II and SMS-EMOA, Eq.~\eqref{eq:nsga-lotz-phase1-1} changes to 
	\[
			\frac{1}{\mu}\cdot 0.1\cdot \frac{1}{n}\cdot \Big(1-\frac{1}{n}\Big)^{n-1}\ge \frac{1}{10e\mu n}.
	\]
	 and Eq.~\eqref{eq:nsga-lotz-phase2-1} changes to 
	\[
	\frac{1}{\mu^2}\cdot 0.9\cdot \frac{n+1-i}{n} \cdot \Big(1-\frac{1}{n}\Big)^n \ge \frac{n+1-i}{2e\mu^2n}.
	\]
	Note that by the uniform parent selection strategy of SMS-EMOA, any solution in the population can be selected with probability $1/\mu$ for reproduction, while by binary tournament selection in NSGA-II, a selected solution needs to compete with the other selected one.  According to the above two equations, the expected number of generations required by the two phases is $O(\mu n^2)$ and 
	$ 	\sum_{i=2}^{n}2e\mu^2n/(n+1-i)=O(\mu^2 n \log n) $, respectively.
	Combining the two phases, the total expected number of generations is  $O(\mu n^2+\mu^2 n\log n)$, implying that the expected number of fitness evaluations is also $O(\mu n^2+\mu^2 n\log n)$ since SMS-EMOA only generates one solution in each generation. 
\end{proof}

The expected number of fitness evaluations of the original SMS-EMOA (without crossover) solving \omm\ and \lotz\ has been shown to be $O(\mu n\log n)$ and $O(\mu n^2)$, respectively, where the population size $\mu\ge n+1$~\cite{zheng2024sms}. Though their analysis does not consider crossover, the running time bounds still hold asymptotically here as the crossover operator is not performed with a constant probability $0.1$ in Algorithm~\ref{alg:sms-emoa}. Note that the size of the Pareto front is $n+1$, thus when $\mu<n+1$, the whole Pareto front cannot be covered.
Therefore, our results in Theorems~\ref{thm:sms-arc-omm} and~\ref{thm:sms-arc-lotz} show that if a constant population size is used for SMS-EMOA having an archive, the expected running time can be accelerated by a factor of $\Theta(n)$.
The main reason for the acceleration is similar to that of NSGA-II. That is, introducing the archive enables SMS-EMOA to only preserve essential solutions for locating the Pareto front, and thus makes the exploration more efficient.

\section{Experiments}

In the previous sections, we have proved that when an archive is used in NSGA-II and SMS-EMOA, the expected number of fitness evaluations for solving \omm\ and \lotz\ can be reduced by a factor of $\Theta(n)$. 
However, as only upper bounds on the running time of the original NSGA-II and SMS-EMOA have been derived, we conduct experiments to examine their actual performance to complement the theoretical results.

\begin{figure}[t!]\centering
    \hspace{1.3em}
	\begin{minipage}[c]{0.85\linewidth}\centering
		\includegraphics[width=1\linewidth]{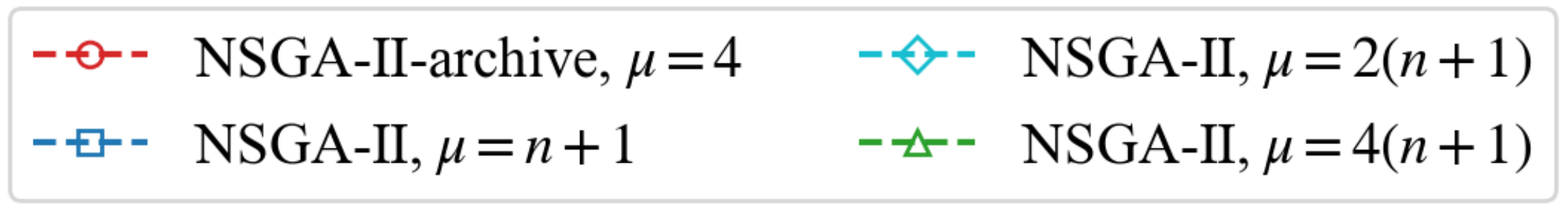}
	\end{minipage}
	\begin{minipage}[c]{0.49\linewidth}\centering
		\includegraphics[width=1\linewidth]{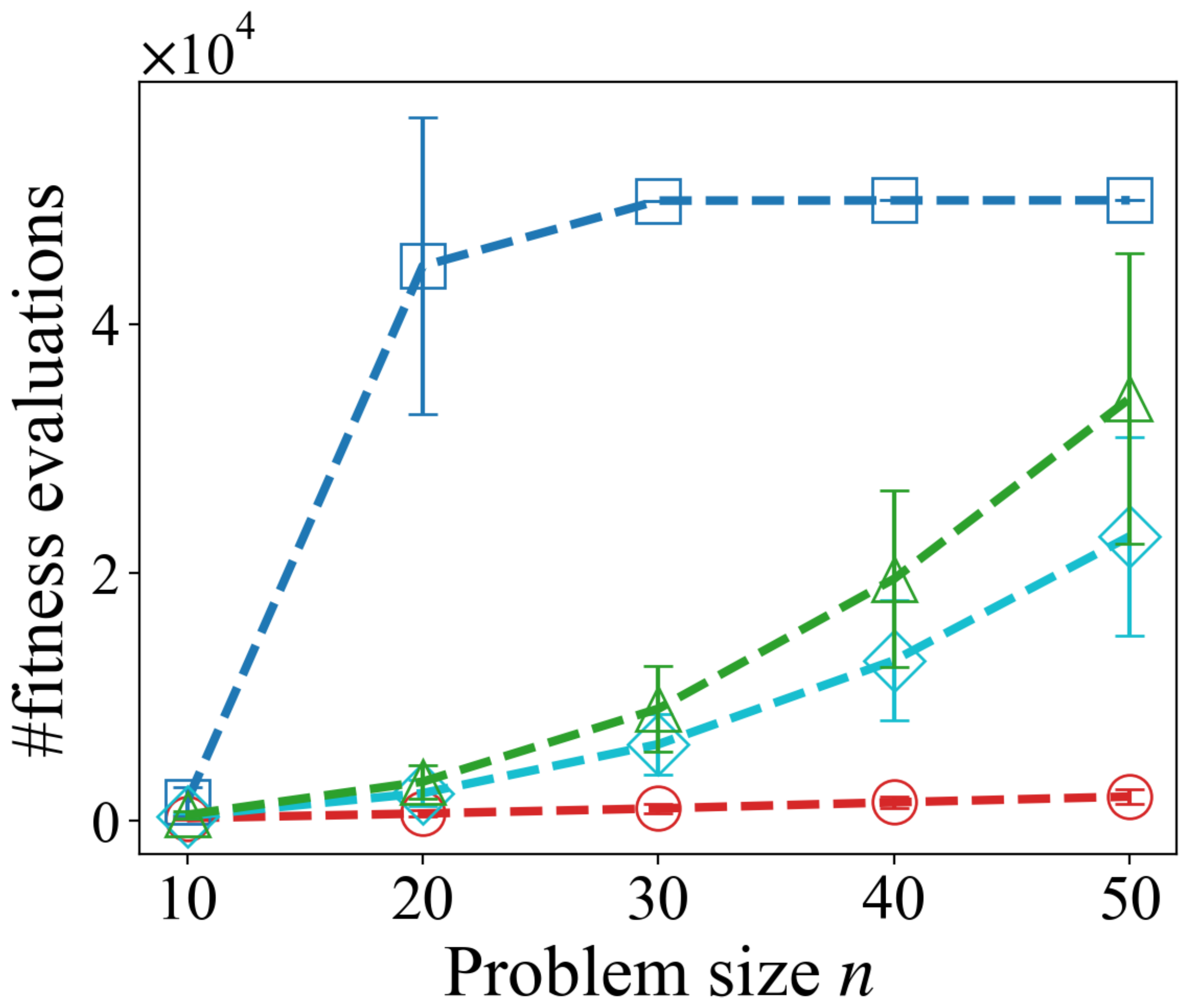}
	\end{minipage}
	\begin{minipage}[c]{0.49\linewidth}\centering
		\includegraphics[width=1\linewidth]{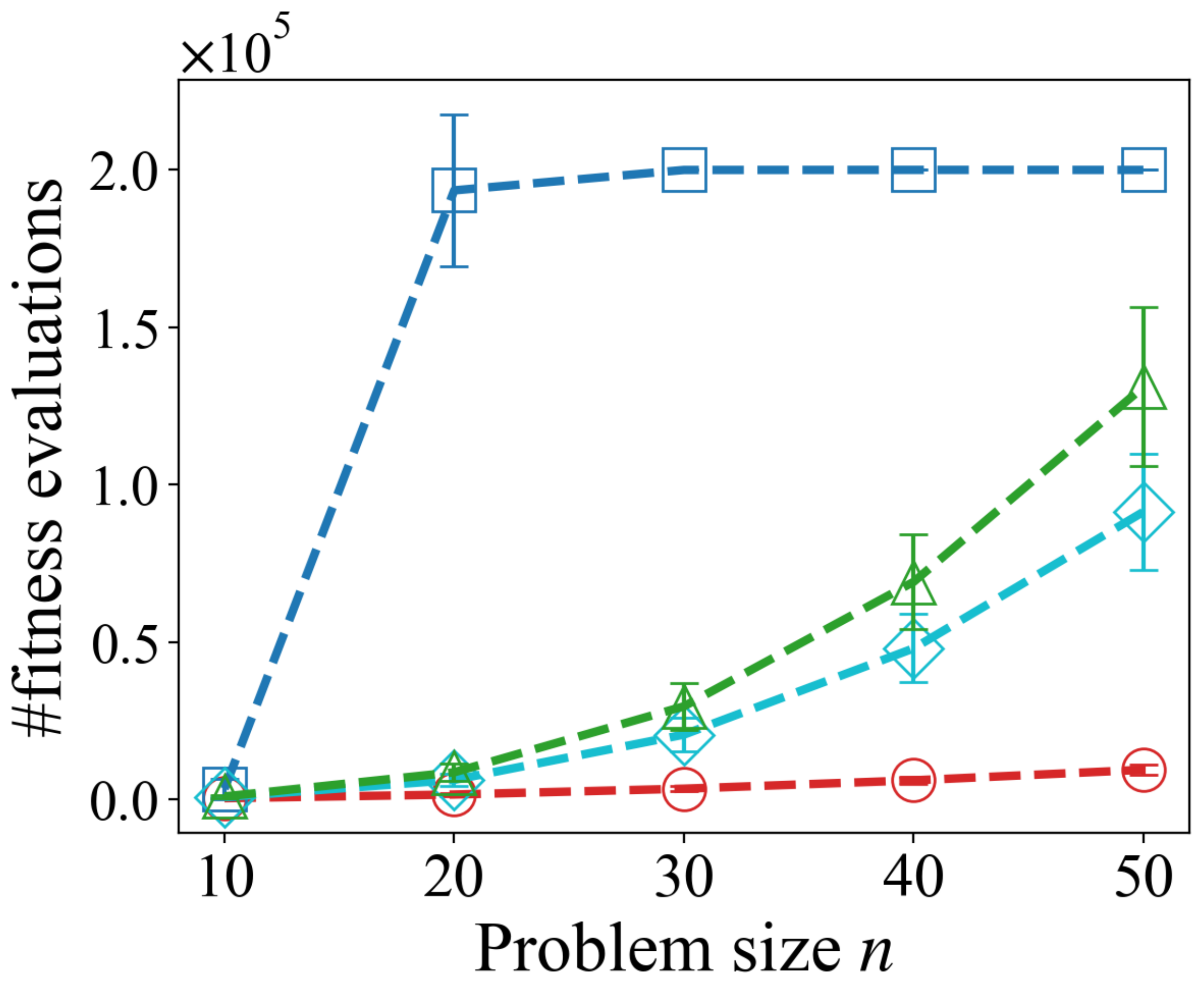}
	\end{minipage}\\\vspace{0.2em}
	\begin{minipage}[c]{1\linewidth}\centering
		\small(a) NSGA-II
	\end{minipage}\\\vspace{0.4em}
    \hspace{1.3em}
        \begin{minipage}[c]{0.88\linewidth}\centering
		\includegraphics[width=1\linewidth]{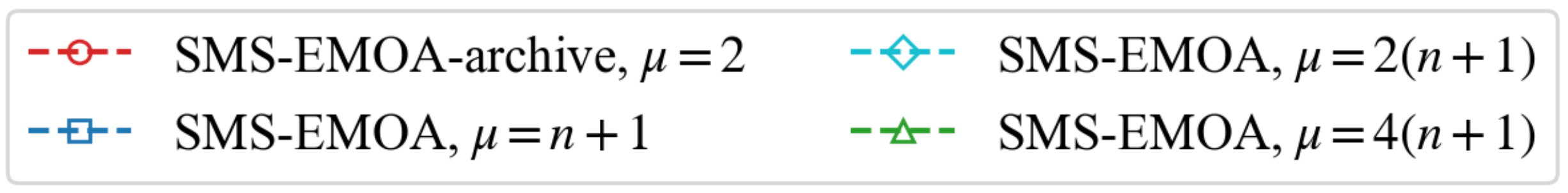}
	\end{minipage}
	\begin{minipage}[c]{0.49\linewidth}\centering
		\includegraphics[width=1\linewidth]{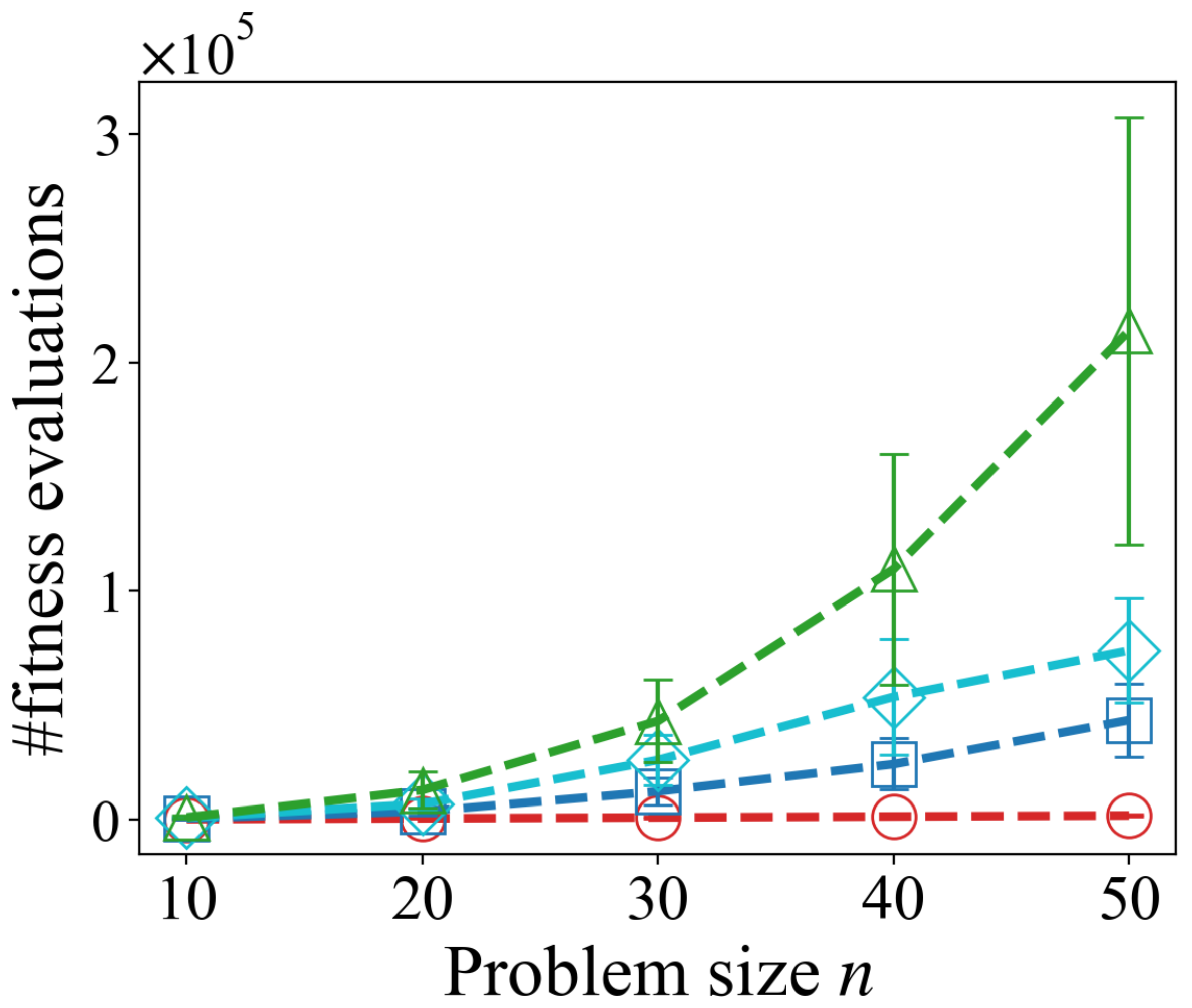}
	\end{minipage}
	\begin{minipage}[c]{0.49\linewidth}\centering
		\includegraphics[width=1\linewidth]{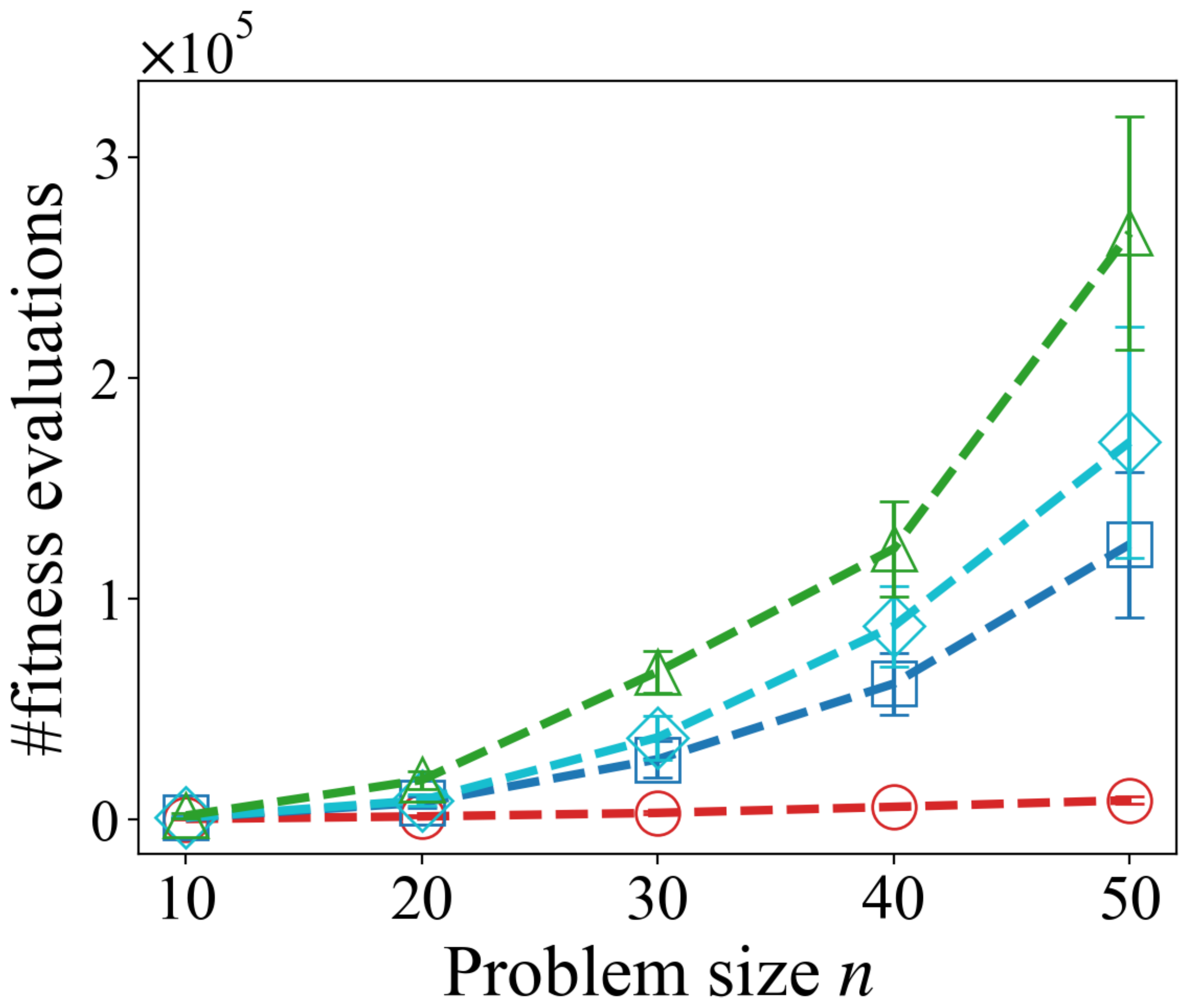}
	\end{minipage}\\\vspace{0.2em}
	\small(b) SMS-EMOA
	\caption{Average number of fitness evaluations of NSGA-II and SMS-EMOA with or without an archive for solving the \omm\ and \lotz\ problems. Left subfigure: \omm; right subfigure: \lotz.}\label{fig:experiment}
\end{figure}

Specifically, we set the problem size $n$ from $10$ to $50$, with a step of $10$. 
For NSGA-II and SMS-EMOA using an archive, the population size $\mu$ is set to $4$ and $2$, respectively, as suggested in Theorems~\ref{thm:nsga-arc-omm}--\ref{thm:sms-arc-lotz}. For the original NSGA-II and SMS-EMOA without an archive, we test three values of $\mu$, i.e., $n+1$, $2(n+1)$, and $4(n+1)$, as suggested by \cite{bian2022better,zheng2023first,zheng2024sms}. Note that $n+1$ is the size of the Pareto front of \omm\ and \lotz. If using a population size $\mu < n+1$, the original algorithms without archive obviously cannot find the Pareto front.
For each $n$, we run an algorithm $1000$ times independently, and record the average number of fitness evaluations until the Pareto front is found. In case where the Pareto front cannot be found in acceptable running time, 
%e.g., NSGA-II using a population size of $n+1$,
we set the maximum number of fitness evaluations to $5\times 10^4$ for \omm\ and $2\times 10^5$ for \lotz.
We can observe from 
Figure~\ref{fig:experiment} that using an archive brings a clear acceleration.

We can also observe that when using a population size of $n+1$, SMS-EMOA performs much better than NSGA-II. The main reason is that for NSGA-II, when two Pareto optimal solutions corresponding to one objective vector (e.g., two solutions corresponding to one boundary point in the Pareto front)  have been found, they both can have fairly large crowding distance, and thus occupy two slots in the population. Then, the population size of $n+1$ is not sufficient to ensure the preservation of objective vectors in the Pareto front. 
However, for SMS-EMOA,  these duplicate  Pareto optimal solutions  have a zero hypervolume contribution, and thus are less preferred, implying that they will not affect the preservation of other objective vectors in the Pareto front. Meanwhile, since SMS-EMOA only removes one solution in each generation, the objective vector corresponding to these duplicate solutions will also be preserved. Thus, a population size of $n+1$ is sufficient for SMS-EMOA to preserve the whole Pareto front.

\section{Conclusion}
In this paper, we perform a first theoretical study for MOEAs with an archive, an increasingly popular practice in the design of MOEAs. Through rigorous running time analysis for NSGA-II and SMS-EMOA solving two commonly studied bi-objective problems, \omm\ and \lotz, we prove that using an archive can allow a constant population size, bringing an acceleration of factor $\Theta(n)$. This is also verified by the experiments. Our results provide theoretical confirmation for the benefit of using an archive to store non-dominated solutions generated during the search process of MOEAs, which has frequently been observed empirically~\cite{fieldsend2003using,bezerra2019}.
In the future, it would be interesting to derive the lower bounds of NSGA-II and SMS-EMOA without using an archive to make the comparison strict. Another interesting direction is studying real-world problems, e.g., multi-objective combinatorial optimization problems.

\section*{Acknowledgements} 
This work was supported by the National Science and Technology Major Project (2022ZD0116600) and National Science Foundation of China (62276124). Chao Qian is the corresponding author. The conference version of this paper has
appeared at IJCAI’24.

\bibliographystyle{named}
\bibliography{ijcai24-archive}

\newpage

\end{document}